%% file: becare_arxiv.tex
\documentclass[letterpaper]{article}
\usepackage[preprint]{aaai2027}
\usepackage[hyphens]{url}
\usepackage{graphicx}
\usepackage{natbib}
\usepackage{caption}
\usepackage{algorithm}
\usepackage{algorithmic}
\usepackage{booktabs}
\usepackage{amsmath}
\usepackage{amssymb}
\usepackage{colortbl}
\usepackage{multirow}
\usepackage{tikz}
\usetikzlibrary{positioning, arrows.meta}
\newcommand{\Full}{\textsc{Full}}
\newcommand{\Cache}{\textsc{Cache}}
\newcommand{\Ncap}{\ensuremath{N_{\mathrm{cap}}}}
\newcommand{\softcap}{\textbf{BeCARE}}
\newtheorem{proposition}{Proposition}
\title{BeCARE: Budgeted Cache Refresh for Diffusion Transformer Acceleration}

\author{
    \resizebox{0.96\textwidth}{!}{%
        Yuhang Zhang\textsuperscript{\rm 1,*},
        Junxiang Qiu\textsuperscript{\rm 2,*},
        Huixia Ben\textsuperscript{\rm 3},
        Zhenhua Tang\textsuperscript{\rm 4},
        Lin Liu\textsuperscript{\rm 2},
        Shuo Wang\textsuperscript{\rm 2},
        Yanbin Hao\textsuperscript{\rm 1,\ensuremath{\dagger}}%
    }
}
\affiliations{
    \resizebox{0.98\textwidth}{!}{%
        \textsuperscript{\rm 1}Hefei University of Technology,
        \textsuperscript{\rm 2}University of Science and Technology of China,
        \textsuperscript{\rm 3}Anhui University of Science and Technology,
        \textsuperscript{\rm 4}University of Macau%
    }\\[-0.05em]
    \resizebox{0.96\textwidth}{!}{%
        zhangyh@mail.hfut.edu.cn, \{qiujx, liulin0725\}@mail.ustc.edu.cn,
        \{huixiaben, shuowang.edu\}@gmail.com, zhenhuat@foxmail.com,
        haoyanbin@hfut.edu.cn%
    }\\[-0.15em]
    {\small \textsuperscript{*}Equal contribution. \textsuperscript{\ensuremath{\dagger}}Corresponding author.}
}

\begin{document}

\maketitle

\begin{abstract}
Training-free feature caching accelerates diffusion transformer (DiT) inference by reusing or forecasting intermediate features. However, fixed schedules make compute predictable but ignore prompt- and timestep-dependent risk, whereas hand-tuned error thresholds adapt locally but leave realized compute difficult to control. To address these limitations, we present \textbf{B}udg\textbf{e}ted \textbf{Ca}che \textbf{Re}fresh (\textbf{BeCARE}), a training-free framework that, given a user-specified cap on the number of full computations, adaptively determines when to refresh the cache during accelerated inference.
Specifically, we first derive an error-amplification profile from the sampler's noise schedule to characterize the varying impact of approximation errors across denoising timesteps. We then combine this profile with prompt-specific extrapolation residuals and cache age to form an accumulated risk score, triggering a refresh when continued caching becomes harmful. Meanwhile, we use the same profile to construct an analytic spending reference and adjust the refresh threshold through feedback, thereby distributing the limited refresh budget over the sampling trajectory. Finally, structural safeguards (a fixed warmup, a budget-derived late-stage reserve, and a maximum cache-reuse length) prevent unreliable extrapolation and premature budget exhaustion. These designs enable prompt-adaptive refresh placement and allow a single calibration to transfer across budget tiers without per-tier tuning.
Experiments on FLUX.1-dev with 200 prompts and measured FLOPs show that our method consistently outperforms representative training-free caching baselines across speedups from $3\times$ to $6\times$. At the main operating point ($3.3\times$ acceleration), it improves paired PSNR by \textbf{2.7} dB over the strongest baseline using no more compute. On SD3.5 Large, it also substantially improves the same Taylor cache engine over fixed-interval scheduling at matched compute.
\end{abstract}

\input{sec/1_introduction}

\input{sec/2_related_work}

\input{sec/3_method}

\input{sec/4_experiments}

\input{sec/5_conclusion}

\clearpage
\bibliography{custom}

\end{document}

%% file: sec/1_introduction.tex
\section{Introduction}

Diffusion models have become a dominant paradigm for high-fidelity image and video generation~\citep{ldm,videoldm}, with diffusion transformers (DiTs) emerging as a scalable backbone~\citep{dit}. However, their strong capability comes with substantial inference cost. Each denoising timestep repeatedly executes a deep stack of Transformer blocks, and the sampling process typically requires dozens of such steps. This heavy sequential computation leads to high latency and FLOPs, posing a major obstacle to the practical deployment of DiTs. Therefore, accelerating DiT inference while preserving generation quality has become an important research problem.

\begin{figure}[t]
\centering
\includegraphics[width=\columnwidth]{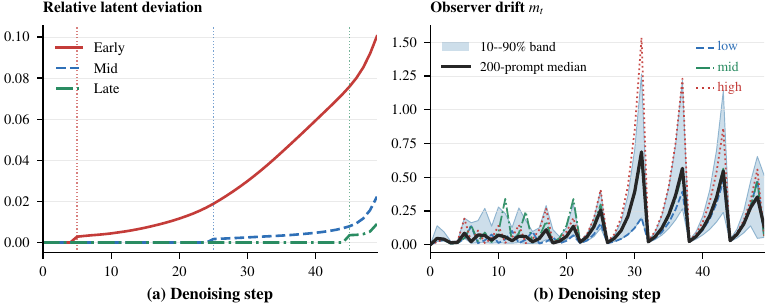}
\caption{
\textbf{Temporal and prompt-dependent structure of cache-refresh signals.}
(a) Relative latent deviation from the uncorrupted trajectory after a single order-2 Taylor intervention early (step 5), mid (25), or late (45).
(b) Online drift signal under the deployed policy (200 prompts): per-step median, 10--90\% band, and three example prompts.}
\label{fig:motivation}
\end{figure}

To alleviate this computational bottleneck, existing DiT acceleration methods mainly follow two directions. The first reduces the number of denoising steps using non-Markovian samplers~\citep{ddim}, high-order solvers~\citep{dpmsolver}, or model distillation~\citep{mgct}. Although effective, these methods typically modify the original sampling process or require additional training. The second reduces the computational cost of each denoising step through feature caching. This strategy performs a full Transformer computation at selected steps, stores the resulting intermediate features, and reuses or forecasts them at subsequent steps. Existing studies have mainly focused on improving \emph{how} skipped computations are approximated, ranging from direct feature reuse~\citep{deepcache, fora, teacache} to feature forecasting, error correction, and speculative verification~\citep{taylorseer,goc,foca,speca}. A complementary line refines \emph{which} components to approximate, at token~\citep{toca,tap}, cluster~\citep{clusca}, frequency~\citep{freqca}, or subspace~\citep{svdcache,hyca} granularity. This paradigm requires no additional training, preserves the pretrained model and sampler configuration, and can be readily applied to existing diffusion transformers.

Despite this rapid progress, existing caching methods still determine cache refresh timing using fixed intervals~\citep{fora,taylorseer} or hand-tuned error thresholds~\citep{teacache,speca}.
The former makes compute predictable but ignores temporal and prompt-dependent risk; the latter adapts locally but is \emph{budget-blind}: its realized compute is an emergent outcome of the threshold, so meeting a deployment target requires searching thresholds by trial and error.
Moreover, these timestep-wise triggers implicitly assume that approximation errors have similar consequences regardless of when they occur. This assumption fails in sequential denoising. 
We probe this directly: on three fixed prompts, we replace the full computation with a single order-2 Taylor approximation at step 5, 25, or 45 of a 50-step trajectory, keep the remaining 49 steps exact, and track the relative latent deviation from the uncorrupted run (Figure~\ref{fig:motivation}~(a)). We find that the early intervention never recovers, because every subsequent exact step evolves from a shifted state: its terminal deviation ($0.100$, red) is roughly $11\times$ that of the late one ($0.009$), with the mid intervention in between. Part of this asymmetry is predictable in advance: the sampler fixes its noise schedule before sampling begins, and this schedule already indicates where along the trajectory errors are most consequential. An offline schedule alone is not sufficient, however. Figure~\ref{fig:motivation}~(b) tracks the drift signal our controller observes online (the deviation of each step's first-block embedding from its Taylor prediction, under the deployed policy) across all 200 benchmark prompts: the per-step median is far from flat, and at a single denoising step the 10--90\% band spans up to $1.07$, so a prediction that is still reliable for one prompt has already drifted badly for another. No fixed refresh pattern serves all prompts equally well. 
Consequently, under the same budget of full computations and identical FLOPs, changing only the placement of these computations can shift output fidelity by several decibels. At matched compute, the bottleneck is where full computations are placed, not how many there are.

We therefore formulate cache refresh scheduling as a feedback-control problem and present \textbf{B}udg\textbf{e}ted \textbf{Ca}che \textbf{Re}fresh (\textbf{BeCARE}), a training-free control layer built on top of a forecasting-based caching engine. 
\textbf{BeCARE} first extracts an \emph{amplification profile} from the sampler's noise schedule, which provides a prior on the impact of approximation errors at different denoising stages. 
During sampling, an observer weights the engine's own extrapolation residuals by this profile and by the current cache age, and accumulates them into a risk score that triggers a refresh once continued caching becomes harmful. Because the residuals are measured online for each prompt, refresh positions are not fixed in advance: prompts with fast-moving features receive their refreshes earlier and closer together, while for easier prompts the controller holds refreshes back for later steps. This addresses the prompt-to-prompt variability shown in Figure~\ref{fig:motivation}~(b).
However, risk-based triggering alone remains budget-blind. 
To address this issue, \textbf{BeCARE} accepts a user-specified cap on the number of full computations and closes a feedback loop around the budget: the cumulative form of the same amplification profile serves as an analytic reference for how much of the cap should have been spent by each step, and a feedback law compares realized spending with this reference and adjusts the refresh threshold, so the cap is neither exhausted early nor left unused late. Finally, a small set of structural guards completes the loop: a fixed initial warmup, plus a reserve for the final steps and a limit on how long a cache may be reused, both following from the cap in closed form. With these pieces the realized compute matches the specification, and one calibration at a single operating point transfers across budget tiers with no per-tier tuning.

The contributions of our method are threefold:
\begin{itemize}
    \item 
    We show that the harm of a cache approximation depends strongly on \emph{when} it occurs, that this asymmetry is anticipated by an amplification profile read directly from the sampler's noise schedule at almost zero cost, and that at fixed FLOPs the placement of full evaluations, not their number, dominates output quality.
    \item 
    We turn the when-to-refresh decision from hand-tuned thresholds into feedback control, in which a single amplification profile drives both a drift-risk observer and a reference-tracking budget controller, governed by one user-facing knob and a small set of structural guards.
    \item 
    On FLUX.1-dev, \softcap{} attains the best full-referenced fidelity among training-free caching baselines at every tested ratio from $3\times$ to $6\times$, while meeting tight budgets exactly. On SD3.5 Large, it substantially improves the same Taylor cache engine over fixed-interval scheduling at matched compute, demonstrating the value of prompt-adaptive refresh placement.
\end{itemize}

%% file: sec/2_related_work.tex
\begin{figure*}[t]
\centering
\includegraphics[width=0.95\textwidth]{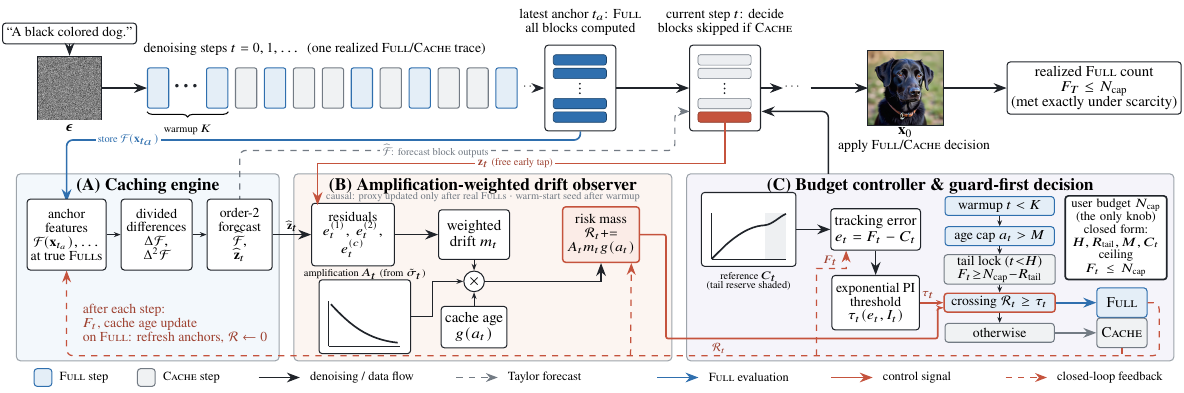}
\caption{\softcap{} overview. Top: one realized \Full{}/\Cache{} trace (warmup in the first $K$ steps). \textbf{(A)} A TaylorSeer-style engine maintains divided differences over dynamically placed anchors and forecasts the skipped block outputs; \textbf{(B)} the observer accumulates amplification- and age-weighted extrapolation residuals into the risk mass $\mathcal{R}_t$; \textbf{(C)} the controller tracks the amplification-shaped reference $C_t$ with an exponential PI threshold under the user budget $\Ncap$, and a guard-first priority chain issues the \Full{}/\Cache{} decision. Red edges carry control signals; dashed red edges close the loop after every step.}
\label{fig:pipeline}
\end{figure*}

\section{Related Work}
Diffusion inference acceleration reduces either the \emph{number} of denoising steps or the \emph{cost} of each step.

\refstepcounter{subsection}\paragraph{\thesubsection\quad Sampling-step reduction. }
The first direction shortens the sampling trajectory itself. Non-Markovian samplers such as DDIM~\citep{ddim} reformulate the reverse process so that steps can be skipped without retraining, and high-order ODE solvers such as DPM-Solver~\citep{dpmsolver} exploit the semi-linear structure of the sampling ODE to reach comparable quality within a few tens of evaluations; distillation and consistency-style transfer~\citep{mgct} instead train students that denoise in very few steps. These methods deliver substantial speedups, but they modify the sampling procedure or require additional training, and quality degrades sharply once the step count becomes very small, a regime our step-reduced baselines quantify directly.

\refstepcounter{subsection}\paragraph{\thesubsection\quad Caching.}
The second direction keeps the sampler and model weights intact and exploits the similarity of intermediate features across adjacent timesteps. Reuse-based methods replay stored module outputs~\citep{deepcache,fora}, later work forecasts or corrects them instead (Taylor extrapolation~\citep{taylorseer}, gradient-based error compensation~\citep{goc,eoc}, linear-multistep and ODE-solver forecasters~\citep{predit,foca}), and a parallel line refines the granularity of caching (token~\citep{toca,tap}, cluster~\citep{clusca}, frequency~\citep{freqca}, subspace~\citep{svdcache,hyca}, hybrid~\citep{hgc}). All of these advances concern \emph{how} the skipped computation is approximated; the decision of \emph{when} to recompute has drawn less attention. Periodic policies recompute at a fixed interval~\citep{fora,taylorseer}; TeaCache accumulates an embedding-modulated change estimate against a hand-tuned threshold~\citep{teacache}; SpeCa verifies Taylor drafts step-by-step and falls back upon rejection~\citep{speca}; RFC uses input prediction error as a proxy for output error~\citep{rfc}. Fixed intervals make compute predictable but cannot respond to local risk; threshold-based triggers adapt locally, yet their compute remains an emergent outcome of a threshold rather than a specification. Neither class jointly treats compute as an input specification and refresh placement as a risk-adaptive decision; existing schedulers also neither reserve computation for later phases nor account for the time-varying consequences of approximation errors under fixed budgets.

%% file: sec/3_method.tex
\section{Method}
\label{sec:method}

\subsection{Preliminaries}
\label{sec:prelim}

\paragraph{Diffusion models.}
A diffusion model generates an image by inverting a gradual noising process~\citep{ddpm}.
In the flow-matching formulation used by modern diffusion transformers~\citep{rectifiedflow,flux2024}, the forward path linearly interpolates between data sample $\mathbf{x}_0$ and noise $\boldsymbol{\epsilon}$ at noise level $\sigma$:
\begin{equation}
\mathbf{x}_\sigma=(1-\sigma)\,\mathbf{x}_0+\sigma\,\boldsymbol{\epsilon},
\qquad \sigma\in[0,1],
\label{eq:forward}
\end{equation}
and generation integrates a learned velocity field backward along a discretized schedule over $T$ steps:
\begin{equation}
\mathbf{x}_{t+1}=\mathbf{x}_t+(\sigma_{t+1}-\sigma_t)\,\mathbf{v}_\theta(\mathbf{x}_t,\sigma_t),
 t=0,\dots,T-1,
\label{eq:reverse}
\end{equation}
where every evaluation of $\mathbf{v}_\theta$ runs a diffusion transformer~\citep{dit} of dozens of blocks; with $T$ in the tens, these full evaluations dominate inference cost.

\paragraph{Feature caching.}
Feature caching exploits similar module features at adjacent steps. \Full{} evaluations execute every block and refresh per-module caches, whereas \Cache{} steps skip the expensive blocks and approximate their outputs from cached state. We use a TaylorSeer-style forecasting engine~\citep{taylorseer}; module indices are omitted for clarity. Let $\mathcal{F}(\mathbf{x}_{t_a})$ be a module's feature at the latest \Full{} anchor $t_a$. A cached step at distance $d$ is predicted by the truncated Taylor expansion
\begin{equation}
\widehat{\mathcal{F}}(\mathbf{x}_{t_a+d})=\mathcal{F}(\mathbf{x}_{t_a})+\textstyle\sum_{i=1}^{O}\Delta^{i}\mathcal{F}(\mathbf{x}_{t_a})\,d^{\,i}/\,i!\,,
\label{eq:taylor}
\end{equation}
with order $O{=}2$, where $\Delta^{i}\mathcal{F}$ is the $i$-th order finite difference of the feature over \Full{} anchors. Because our scheduler places anchors dynamically, their spacing is not constant, so the engine maintains \emph{divided} differences, updated recursively from the two most recent anchors $t_a$ and $t_{a}'$:
\begin{equation}
\Delta^{i}\mathcal{F}(\mathbf{x}_{t_a})=
\begin{cases}
\mathcal{F}(\mathbf{x}_{t_a}), & i=0,\\
\frac{\Delta^{i-1}\mathcal{F}(\mathbf{x}_{t_a})
-\Delta^{i-1}\mathcal{F}(\mathbf{x}_{t_a'})}
{t_a-t_a'}, & i>0.
\end{cases}
\label{eq:findiff}
\end{equation}

For anchors uniformly spaced at interval $N$, Eq.~\ref{eq:findiff} recovers the standard TaylorSeer differences with $1/N^{i}$ normalization. Because \Cache{} steps skip the expensive blocks, per-image compute is essentially linear in the number of \Full{} evaluations (measured fit in the technical supplement).

\paragraph{Problem statement.}
Given a compute budget $\Ncap$, a causal online policy must choose at each step whether to spend one of at most $\Ncap$ \Full{} evaluations or to cache. We formulate scheduling as \emph{reference tracking}: the policy maintains a cumulative spending target $C_t$ with $C_{T-1}=\Ncap$, compares it with the realized count $F_t$, and modulates a risk threshold to concentrate full evaluations where needed while keeping cost predictable. The budget is a \emph{soft ceiling}: $F_t\le\Ncap$ is a hard cap, but the policy need not spend up to it.

Figure~\ref{fig:pipeline} gives an overview of \softcap{}. 
At every step, an \emph{observer} turns the engine's extrapolation residual into an amplification- and age-weighted risk mass; a \emph{controller} derives an analytic spending reference from the budget and adjusts the trigger threshold by exponential PI feedback; and \emph{structural guards} bound failure modes that feedback cannot exclude, with safety rules outranking budget rules. A step runs \Full{} when the accumulated risk crosses the threshold or a guard expires; otherwise it caches.

\subsection{Amplification-Weighted Risk Observer}
\label{sec:observer}

The observer estimates how dangerous it is to keep caching at step $t$ as the product of three factors, each available at run time: how strongly an error at this step propagates (a schedule-derived prior), how far the current prompt's trajectory has drifted, and how stale the cache already is.

\paragraph{Schedule-derived amplification prior.}
Every diffusion sampler fixes its noise schedule before sampling begins. On FLUX, for instance, the sampler assigns step $t$ a shifted noise level $\tilde{\sigma}_t$; we use its mean-normalized, floored profile
\begin{equation}
A_t=\max(\tilde{\sigma}_t,f)\Big/\Big(\textstyle\sum_{j=0}^{T-1}\max(\tilde{\sigma}_j,f)/T\Big),
\label{eq:amp}
\end{equation}
as a \emph{prior} on how strongly a unit of approximation error at step $t$ propagates through the remaining trajectory: early errors are reshaped by the entire remaining denoising process, late errors remain mostly local. The shape of $A_t$ is read from the sampler's own schedule rather than tuned; the floor $f$ keeps late-step weights from vanishing (measured drift does not vanish there), and mean normalization keeps $A_t$ from rescaling the global threshold (calibrated values in the technical supplement). $A_t$ is a heuristic weighting, not a verified pointwise damage profile.

\paragraph{Causal Taylor-residual signal.}
The observer measures how far the engine's own extrapolation has already drifted. Let $\mathbf{z}_t$ be the image-token input to the first Transformer block (produced by a lightweight embedding that every step computes anyway), and let $\widehat{\mathbf{z}}_t$ be its order-2 Taylor prediction built from previous \Full{} anchors and updated \emph{only} after a real \Full{} evaluation, so the signal is causal and prompt-specific. We form three residuals between $\mathbf{z}_t$ and its prediction:
\begin{equation}
\begin{aligned}
e^{(1)}_t&=\frac{\|\mathbf{z}_t-\widehat{\mathbf{z}}_t\|_1}{\|\mathbf{z}_t\|_1+\epsilon_{\mathrm{n}}},\qquad
e^{(2)}_t=\frac{\|\mathbf{z}_t-\widehat{\mathbf{z}}_t\|_2}{\|\mathbf{z}_t\|_2+\epsilon_{\mathrm{n}}},\\
e^{(c)}_t&=\mathrm{clip}\big(1-\cos(\mathbf{z}_t,\widehat{\mathbf{z}}_t),\,0,\,2\big),
\end{aligned}
\label{eq:residuals}
\end{equation}
where $\epsilon_{\mathrm{n}}$ is a small numerical stabilizer. The drift signal fuses the three floored residuals with fixed weights,
\begin{equation}
m_t=\sum_{j} w_j \max\big(e^{(j)}_t,\,\epsilon_{\mathrm{f}}\big),
\qquad j\in\{1,2,c\},
\label{eq:mt}
\end{equation}
where $\epsilon_{\mathrm{f}}$ is a small numerical floor and the weights $w_j$ are engineering scales that bring the three terms to comparable magnitude (the raw cosine gap is orders of magnitude smaller than the norm residuals), fixed once during calibration and frozen (values in the technical supplement).

\paragraph{Risk-mass accumulation.}
Caching for longer accumulates error rather than resetting it, so the decision variable is an integral. With prospective cache age $a_t$ (the age the cache would reach if step $t$ caches),
\begin{equation}
r_t=A_t\, m_t\, g(a_t),\qquad \mathcal{R}_t=\mathcal{R}_{t-1}+r_t,
\label{eq:risk}
\end{equation}
where $g(\cdot)$ is a monotone cache-age multiplier estimated once from a development-set audit of true cache errors (per-age means normalized to age 1; values in the technical supplement). The accumulated mass $\mathcal{R}_t$ thus grows with the prompt's measured drift, with the phase-dependent amplification, and with cache staleness.

\paragraph{Warm start, trigger, and reset.}
The first $K$ warmup steps are forced \Full{} (a structural guard, below). Warmup predictions use only true \Full{} anchors (missing step-$0$ residuals take $\epsilon_{\mathrm{f}}$), so the accumulator exits warmup seeded with the running mean $\mathcal{R}_{K-1}=\mathrm{mean}(r_0,\dots,r_{K-1})$: high-drift prompts reach their first crossing slightly earlier, while low-drift prompts are unchanged. A step triggers a \emph{crossing} when $\mathcal{R}_t\ge\tau_t$ (threshold defined below); after every non-warmup \Full{}, the accumulator resets to zero. The asymmetry is deliberate: re-seeding from cached-segment observations injects extrapolation error and collapses sequence diversity (negative result in the technical supplement).

\subsection{Soft-Budget Feedback Controller}
\label{sec:controller}

\paragraph{Budget ledger.}
The controller partitions the budget in closed form. With warmup length $K{=}4$ and tail reserve $R_{\mathrm{tail}}(\Ncap)=\mathrm{round}(4\Ncap/15)$ (the calibrated main-tier value scaled proportionally and withheld for the endgame), the freely allocable front budget is $P=\Ncap-K-R_{\mathrm{tail}}$, and the release horizon is
\begin{equation}
H=\mathrm{clip}\big(\lfloor \rho_H T\rfloor{+}1,\;K{+}1,\;T{-}1\big),
\label{eq:horizon}
\end{equation}
with the fraction $\rho_H$ fixed once for all experiments (numerical values in the experimental setup).

\paragraph{Analytic reference curve.}
The spending target is piecewise analytic, shaped by the same amplification prior that weights the observer:
\begin{equation}
C_t=\begin{cases}
t+1, & t<K,\\[2pt]
K+P\,\frac{\sum_{j=K}^{t}A_j}{\sum_{j=K}^{H-1}A_j}, & K\le t<H,\\[4pt]
K+P+R_{\mathrm{tail}}\big(\frac{t-H+1}{T-H}\big)^{q}, & H\le t<T,
\end{cases}
\label{eq:refcurve}
\end{equation}
with tail power $q{=}2$ (calibrated once and shared across tiers), followed by a running cumulative maximum and the endpoint pin $C_{T-1}=\Ncap$. The front segment allocates $P$ evaluations along the CDF of $A_t$ (denser where amplification is high), and the tail segment releases the reserved $R_{\mathrm{tail}}$ toward the end of the trajectory; $C_t$ is a real-valued \emph{reference}, not a prescribed schedule that spending must match exactly.

\paragraph{Exponential PI feedback.}
Let $F_t$ be the \Full{} count before the current decision. The threshold follows an exponential proportional--integral law on the tracking error $e_t$ and its integral $I_t$:
\begin{equation}
e_t=F_t-C_t,\qquad
\tau_t=\tau_0\exp\!\big(K_p e_t+K_i I_t\big),
\label{eq:pi}
\end{equation}
with the exponent clipped for numerical stability. Spending ahead of the reference suppresses crossings, falling behind invites them, and the exponential form keeps the threshold positive. A \emph{hard gate} additionally forbids crossings outright whenever spending runs more than a quarter evaluation ahead of the reference; safety guards (below) are exempt. The base threshold $\tau_0$ (the mean amplification-weighted drift per allocatable slot on the development prompts, scaled by a safety factor) and the gains $(K_p,K_i)$ are calibrated once and frozen across all budget tiers (values in the technical supplement).

\paragraph{Soft-ceiling semantics.}
$\Ncap$ is a ceiling, not a quota: under tight budgets demand exceeds supply and every run hits it exactly, whereas under a loose one low-risk prompts may naturally stop below it; no terminal catch-up evaluation is inserted, and realized FLOPs are metered as spent (empirical behavior in Experiments).

\input{tab/main_results}

\subsection{Structural Guards and Cross-Tier Scaling}
\label{sec:guards}

Feedback alone cannot prevent early budget exhaustion or unbounded extrapolation. Three structural rules address these failures: a fixed warmup required for engine initialization, plus an age cap and tail reserve derived from $(\Ncap,T)$ in closed form.

\paragraph{Derivative warmup.}
Steps $0,\dots,K{-}1$ are forced \Full{}. Order-2 extrapolation requires a history of true anchors to initialize its difference factors; the specific choice $K{=}4$ is validated by ablation rather than uniquely implied by the engine order (see Experiments).

\paragraph{Age cap.}
A cache may not be extended past
\begin{equation}
M(\Ncap)=\left\lfloor\frac{T-K}{\Ncap-K}\right\rfloor+1,
\label{eq:agecap}
\end{equation}
consecutive cached steps (Eq.~\ref{eq:agecap} requires $\Ncap>K$, which every tier satisfies); expiry forces a \Full{}. The numerator is the per-slot horizon share after warmup, and the $+1$ leaves slack so that the guard does not itself degenerate into a uniform schedule; this yields $M{=}3,5,6,8$ at $\Ncap{=}20,15,12,10$.

\paragraph{Tail reserve.}
Before horizon $H$, once ordinary crossings have consumed $\Ncap-R_{\mathrm{tail}}$ evaluations, further crossings are locked and the step caches, preserving capacity for the tail; after $H$, the reserve is released along the tail segment of Eq.~\ref{eq:refcurve}. The reserve is a reason to \emph{cache}, not a source of full evaluations: at $\Ncap{=}15$, the realized budget averages $4$ warmup, $7.7$ crossing, and $3.3$ age-cap evaluations, and the reserve never appears as a trigger. It prevents early high-risk spending from stripping the endgame of anchors. Removing it at identical \Full{} count and FLOPs is the most damaging ablation we test (see Experiments).

\paragraph{Guard-first priority.}
By definition, the decision priority is \emph{warmup} $\to$ \emph{age-cap} $\to$ \emph{tail lock} $\to$ \emph{crossing} $\to$ \emph{cache}: the age-cap safety guard outranks the budget lock. Reversing them violates the advertised age bound at low budgets (a 13-step pure-extrapolation gap against a promised 9; complete decision rule and traced example in the technical supplement). The two guard contracts are exact:

\begin{proposition}[Controller invariants]
\label{prop:invariants}
Under the guard-first rule, every realized trajectory satisfies, for all $t$: (i) $F_t\le\Ncap$; and (ii) whenever $F_t<\Ncap$, at most $M(\Ncap)$ consecutive steps cache, i.e., successive \Full{} anchors are at most $M{+}1$ steps apart.
\end{proposition}
\noindent\emph{Proof:} in the technical supplement, alongside the complete decision rule.

Quality is \emph{not} covered by Proposition~\ref{prop:invariants}; it is established empirically in the Experiments section.

\paragraph{Provenance and overhead.}
The technical supplement tabulates the frozen policy: each quantity is the user budget $\Ncap$, host-given ($T$, sampler schedule, engine order), derived in closed form ($H$, $R_{\mathrm{tail}}$, $M$, $C_t$), or calibrated once on the development set. Thus the method is \emph{once-calibrated and budget-transferable}: calibration at F15 extends across the primary operating range through the same closed-form rules without per-tier search. It is not parameter-free; calibration is specific to a backbone--sampler pair, and we do not claim validation beyond $T{=}50$. At run time, the observer reads $\mathbf{z}_t$, which every step already computes. Its proxy requires one Taylor extrapolation of a single early-layer tensor, while the controller and guards use $O(1)$ arithmetic per step. Both are small relative to the dozens of blocks in a \Full{} evaluation, and per-image FLOPs are deterministic given a trace (wall-clock notes are in the supplement).

%% file: tab/main_results.tex
\begin{table*}[t]
\centering
\small
\setlength{\tabcolsep}{5pt}
\begin{tabular}{l cc cc cccc}
\toprule
\textbf{Method} & \textbf{FLOPs(T)}$\downarrow$ & \textbf{Speedup}$\uparrow$ & \textbf{ImageReward}$\uparrow$ & \textbf{CLIP}$\uparrow$ & \textbf{Clean-FID}$\downarrow$ & \textbf{PSNR}$\uparrow$ & \textbf{SSIM}$\uparrow$ & \textbf{LPIPS}$\downarrow$ \\
\midrule
Full FLUX, 50 steps & 3719.50 & 1.00$\times$ & 0.9957 & 27.77 & 0.00 & $\infty$ & 1.0000 & 0.0000 \\
Full FLUX, 25 steps & 1859.75 & 2.00$\times$ & 0.9754 & 27.70 & 86.14 & 13.957 & 0.5497 & 0.4847 \\
Full FLUX, 10 steps & 743.90 & 5.00$\times$ & 0.8214 & 27.59 & 107.67 & 12.775 & 0.5006 & 0.5679 \\
\midrule
FORA ($N{=}2$) & 1860.32 & 2.00$\times$ & 0.9361 & 27.44 & 104.96 & 10.684 & 0.4349 & 0.6558 \\
TaylorSeer ($N{=}2$, $O{=}2$) & 1934.69 & 1.92$\times$ & \textbf{0.9993} & 27.65 & 68.87 & 16.849 & 0.6179 & 0.3889 \\
\rowcolor{gray!20}
\softcap{} (F20) & 1486.99 & 2.50$\times$ & 0.9945 & \textbf{27.83} & \textbf{67.28} & \textbf{17.163} & \textbf{0.6222} & \textbf{0.3805} \\
\midrule
FORA ($N{=}3$) & 1265.38 & 2.94$\times$ & 0.8891 & 27.59 & 108.37 & 10.904 & 0.4417 & 0.6565 \\
TaylorSeer ($N{=}3$, $O{=}2$) & 1339.75 & 2.78$\times$ & \textbf{1.0210} & 27.78 & 78.45 & 14.960 & 0.5692 & 0.4408 \\
ClusCa ($N{=}4$, $O{=}1$, $K{=}16$) & 1045.58 & 3.56$\times$ & 1.0061 & 27.71 & 84.99 & 13.992 & 0.5387 & 0.4796 \\
SpeCa & 1148.60 & 3.24$\times$ & 0.9834 & 27.70 & 87.55 & 13.698 & 0.5353 & 0.4943 \\
\rowcolor{gray!20}
\softcap{} (F15) & 1116.64 & 3.33$\times$ & 0.9953 & \textbf{27.84} & \textbf{70.46} & \textbf{16.717} & \textbf{0.6074} & \textbf{0.3924} \\
\midrule
FORA ($N{=}4$) & 967.91 & 3.84$\times$ & 0.8697 & 27.61 & 111.21 & 11.118 & 0.4492 & 0.6542 \\
TaylorSeer ($N{=}5$, $O{=}2$) & 893.54 & 4.16$\times$ & 0.9607 & 27.72 & 90.49 & 13.642 & 0.5244 & 0.5006 \\
ClusCa ($N{=}5$, $O{=}2$, $K{=}16$) & 897.03 & 4.15$\times$ & 0.9687 & 27.71 & 90.69 & 13.645 & 0.5243 & 0.5007 \\
SpeCa & 940.89 & 3.95$\times$ & 0.9776 & 27.69 & 91.87 & 13.285 & 0.5116 & 0.5193 \\
\rowcolor{gray!20}
\softcap{} (F12) & 893.54 & 4.16$\times$ & \textbf{1.0049} & \textbf{27.97} & \textbf{76.23} & \textbf{15.994} & \textbf{0.5729} & \textbf{0.4221} \\
\midrule
FORA ($N{=}5$) & 744.81 & 4.99$\times$ & 0.8094 & 27.62 & 114.08 & 11.252 & 0.4459 & 0.6569 \\
TaylorSeer ($N{=}6$, $O{=}2$) & 744.81 & 4.99$\times$ & 0.9801 & 27.71 & 114.54 & 10.204 & 0.3898 & 0.6877 \\
ClusCa ($N{=}6$, $O{=}2$, $K{=}16$) & 748.48 & 4.97$\times$ & \textbf{0.9860} & 27.74 & 98.14 & 13.110 & 0.4926 & 0.5386 \\
SpeCa & 745.33 & 4.99$\times$ & 0.9482 & \textbf{28.00} & 102.88 & 12.568 & 0.4730 & 0.5676 \\
\rowcolor{gray!20}
\softcap{} (F10) & 744.81 & 4.99$\times$ & 0.9809 & 27.83 & \textbf{86.18} & \textbf{14.427} & \textbf{0.5267} & \textbf{0.4721} \\
\midrule
FORA ($N{=}7$) & 596.07 & 6.24$\times$ & 0.6626 & 27.22 & 120.53 & 11.468 & 0.4511 & 0.6714 \\
TaylorSeer ($N{=}8$, $O{=}2$) & 596.07 & 6.24$\times$ & 0.8758 & 27.49 & 120.66 & 10.019 & 0.3660 & 0.7152 \\
ClusCa ($N{=}8$, $O{=}2$, $K{=}16$) & 599.93 & 6.20$\times$ & 0.8513 & 27.57 & 111.72 & 12.362 & 0.4487 & 0.5983 \\
SpeCa & 604.32 & 6.15$\times$ & 0.9207 & 28.00 & 114.84 & 12.077 & 0.4313 & 0.6245 \\
\rowcolor{gray!20}
\softcap{} (F8)$^\dagger$ & 596.07 & 6.24$\times$ & \textbf{0.9960} & \textbf{28.05} & \textbf{94.02} & \textbf{13.983} & \textbf{0.4985} & \textbf{0.5079} \\
\bottomrule
\end{tabular}
\caption{Quantitative comparison of text-to-image generation on FLUX.1-dev. \textbf{Bold} marks the best accelerated result within each compute regime. $^\dagger$Per-tier instantiation in the technical supplement.}
\label{tab:main}
\end{table*}

%% file: sec/4_experiments.tex
\section{Experiments}
\label{sec:experiments}

\subsection{Setup}

We evaluate 50-step FLUX.1-dev generation ($1024\times1024$) on the standard $200$-prompt DrawBench benchmark~\citep{imagen}. With identical prompt seeds, all one-time calibration uses only the first $20$ prompts; the remaining $180$ are untouched by development decisions, and all main-table conclusions hold on both splits. For each prompt and seed, the Full 50-step output is the reference for paired PSNR, SSIM, and LPIPS~\citep{lpips}; clean-FID~\citep{fid,cleanfid} compares the generated and Full-reference sets. We call these \emph{Full-referenced} metrics. ImageReward~\citep{imagereward} and CLIPScore~\citep{clipscore} are reference-free. We compare SpeCa~\citep{speca}, TaylorSeer~\citep{taylorseer}, FORA~\citep{fora}, and ClusCa~\citep{clusca} at the regimes in Table~\ref{tab:main}, plus 25- and 10-step Full sampling. We repeat this protocol on Stable Diffusion 3.5 Large (SD3.5 Large)~\citep{rectifiedflow} with the same prompts and seeds, separately measured FLOPs, and the same order-2 Taylor cache engine for \softcap{} and TaylorSeer. Speedups are measured-FLOPs reductions from each backbone's Full 50-step run, not latency; prompt-varying FLOPs are averaged over all $200$ prompts. Supplementary Appendices~B, C, and~F give configurations, FLOPs accounting, and complete cross-backbone results.

\subsection{Main Results}

\begin{figure}[t]
\centering
\includegraphics[width=0.98\columnwidth]{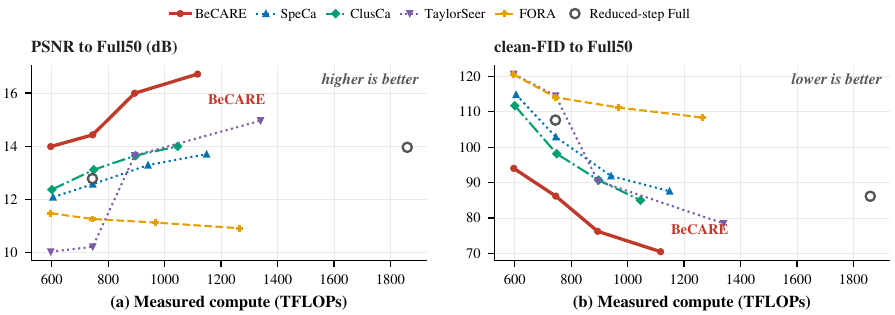}
\caption{Full-referenced quality versus measured compute on 200 prompts. \softcap{} leads the training-free baselines in PSNR (a) and clean-FID (b) at matched FLOPs, and degrades gracefully toward higher acceleration where competitor curves drop steeply.}
\label{fig:speedcurve}
\end{figure}

\input{tab/sd35_results}

Table~\ref{tab:main} compares the \softcap{} configurations (one calibration with only $\Ncap$ changed for F10--F20; F8 is a separately adapted extreme endpoint) against all baselines grouped by compute regime; Figure~\ref{fig:speedcurve} plots the corresponding quality--compute curves.

\paragraph{\softcap{} leads at every tested tier.}
\softcap{} gives the best Full-referenced fidelity among the compared training-free baselines at every tested budget. At F15 ($1116.6$~TF, $3.33\times$), it reaches $16.72$~dB PSNR and $70.46$ FID; its paired PSNR margin over the strongest external baseline using no more compute (ClusCa) is $+2.73$~dB. Near $894$~TF, F12 gains $+2.35$~dB over the best external method, while at identical $744.8$~TF ($4.99\times$), F10 gains $+3.18$~dB over the stronger of TaylorSeer and FORA. F12, F10, and F8 exactly match uniform baselines in FLOPs, so their margins arise from \emph{where} Full evaluations are placed. F10 also exceeds the 25-step Full baseline by $0.47$~dB with $40\%$ of its compute. Figure~\ref{fig:qualitative} confirms the trend visually: at the panel's least compute, \softcap{} stays closest to the reference.

\paragraph{Statistical significance.}
Paired prompt-level bootstrap against the strongest external baseline using no more compute (ClusCa, $n{=}200$) at the main operating point gives $\Delta$PSNR $+2.726$ with $95\%$ CI $[+2.263,+3.202]$; the SSIM and LPIPS intervals likewise exclude zero with $P(\mathrm{better}){=}1.0$ (protocol and full intervals in the technical supplement).

\paragraph{The knee is at F15.}
The quality--budget curve is strongly concave: each additional Full evaluation gains roughly $0.5$~dB from F10 to F15, but only $0.09$~dB from F15 to F20. This knee motivates F15 as the main operating point.

\paragraph{Risk-aware refresh placement restores fidelity on SD3.5 Large.}
At high acceleration, the fixed-interval TaylorSeer baseline suffers severe fidelity loss on SD3.5 Large (Table~\ref{tab:sd35main}). This exposes a limitation of fixed-step refresh: predetermined anchors cannot respond to prompt- and step-specific extrapolation risk. \softcap{} retains the same order-2 Taylor extrapolator and changes only where Full refreshes occur. At matched $4.094\times$ compute ($764.79$ vs. $764.77$~TF), it gains $2.76$~dB PSNR and $0.100$ SSIM, reduces LPIPS by $0.173$ and clean-FID by $26.70$, and raises ImageReward from $0.2195$ to $0.6872$ (Supplementary Appendix~F). Thus risk-aware placement alone restores substantial quality without modifying the Taylor extrapolator.

\begin{figure}[t]
\centering
\includegraphics[width=\columnwidth]{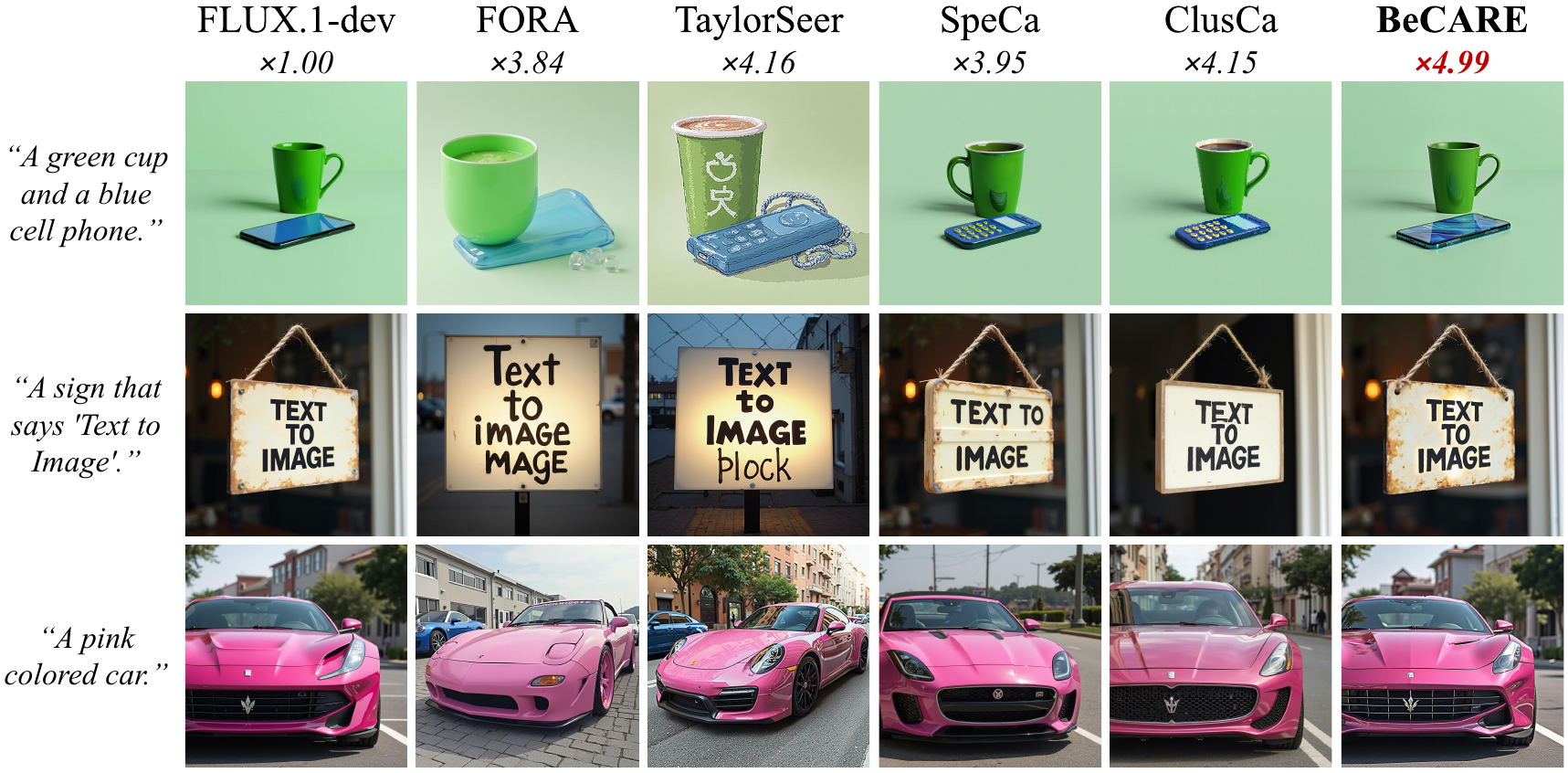}
\caption{\textbf{Qualitative comparison on FLUX.1-dev.}
\mbox{BeCARE} operates at $4.99\times$ acceleration, whereas FORA, TaylorSeer,
SpeCa, and ClusCa operate at $3.84$--$4.16\times$.
Despite less computation, BeCARE better preserves the reference
composition, text rendering, and object appearance.}
\label{fig:qualitative}
\end{figure}

\input{tab/ablation_bootstrap}

\subsection{Component Ablations}
\label{sec:exp-ablation}

\paragraph{Leave-one-out at the main operating point.}
Table~\ref{tab:ablation} removes one component at a time at F15; reruns reproduce the recorded outputs, validating the switches. The tail reserve is strongest ($-6.55$~dB) and fails at \emph{identical} Full count and FLOPs, so its value is placement rather than compute savings. PI feedback and the fourth warmup anchor each contribute about $2$~dB, while removing amplification costs $0.56$~dB at the knee. The warm-start seed and age cap have smaller effects; without the age cap, the policy underspends to $14.5$ Full evaluations. A 200-prompt control isolates placement end to end: at the same 15-Full budget and FLOPs, uniform spacing reaches $13.72$~dB and no tail reserve $13.31$~dB, versus \softcap{}'s $16.72$~dB. Hard-coded prefire steps underperform warm start in the worst case, while re-seeding after every \Full{} collapses sequence diversity (technical supplement).

\paragraph{Budget interaction and guard order.}
At F10, only ${\sim}3$ Full evaluations remain adaptive after warmup, guards, and reserve. On 20 development prompts, removing amplification changes PSNR by $+0.26$~dB with intervals crossing zero, whereas removing PI costs $0.52$~dB. Giving the age cap priority over the tail lock shortens F10's worst pure-extrapolation gap from $13$ to $9$ steps and improves all five metrics at unchanged FLOPs; F12 remains within its preregistered non-inferiority bounds (Supplementary Appendices~A and~E).

\subsection{Adaptivity and Soft-Budget Behavior}

\paragraph{Trajectories are prompt-adaptive.}
At F15, the 200 prompts realize $41$ Full-step sequences; the most common covers only $39/200$ prompts. Observer crossings trigger $7.7$ of the $15$ Full evaluations on average (the rest: $4$ warmup, $3.3$ age-cap). Diversity rises from $4$ to $103$ sequences between F8 and F20 as guards release their share, but contracts sharply at F12/F10. The technical supplement gives tracking traces and placement rasters.

\paragraph{The ceiling is soft.}
F15, F12, F10, and F8 hit their budgets on $200/200$ prompts without a terminal catch-up rule. Under surplus (F20), $4/200$ low-risk prompts stop at $19$ because their next age-cap expiry lies beyond the horizon and final risk remains below threshold. A preselected $15$-prompt hard subset carries $21\%$ more mean drift than the calibration subset, yet its mean crossing count remains close ($7.73$ vs. $7.35$) as triggered and guard-forced spending trade off.

%% file: tab/sd35_results.tex
\begin{table}[!t]
\centering
\scriptsize
\renewcommand{\arraystretch}{0.9}
\setlength{\tabcolsep}{2pt}
\begin{tabular}{@{}lcccccc@{}}
\toprule
\textbf{Method} & \textbf{FLOPs(T)}$\downarrow$ & \textbf{Speed}$\uparrow$ &
\textbf{PSNR}$\uparrow$ & \textbf{SSIM}$\uparrow$ &
\textbf{LPIPS}$\downarrow$ & \textbf{IR}$\uparrow$ \\
\midrule
Full SD3.5, 50 steps & 3130.85 & $1.000\times$ & $\infty$ & 1.0000 & 0.0000 & 1.0827 \\
\midrule
FORA ($N{=}3$) & 1138.16 & $2.751\times$ & 10.6779 & 0.4636 & 0.5917 & \textbf{0.9925} \\
TaylorSeer ($N{=}3$) & 1138.36 & $2.750\times$ & 11.3666 & 0.5295 & 0.5107 & 0.7922 \\
\rowcolor{gray!20}
\softcap{} (F18) & 1080.78 & $2.897\times$ & \textbf{14.4459} & \textbf{0.6540} & \textbf{0.3533} & 0.9832 \\
\rowcolor{gray!20}
\softcap{} (F20) & 1256.06 & $2.493\times$ & 14.6116 & 0.6641 & 0.3433 & 0.9968 \\
\midrule
FORA ($N{=}4$) & 889.08 & $3.521\times$ & 10.0990 & 0.4013 & 0.6608 & 0.9206 \\
TaylorSeer ($N{=}4$) & 889.30 & $3.521\times$ & 10.1756 & 0.4678 & 0.5964 & 0.4672 \\
\rowcolor{gray!20}
\softcap{} (F15) & 948.78 & $3.300\times$ & \textbf{13.6103} & \textbf{0.6071} & \textbf{0.3971} & \textbf{0.9303} \\
\bottomrule
\end{tabular}
\caption{Quantitative comparison on SD3.5 Large. \textbf{Bold} marks the best accelerated result per compute regime; complete results are in Supplementary Appendix~F.}
\label{tab:sd35main}
\end{table}

%% file: tab/ablation_bootstrap.tex
\begin{table}[t]
\centering
\small
\setlength{\tabcolsep}{3.5pt}
\begin{tabular}{lrrrr}
\toprule
Variant & Full & PSNR$\uparrow$ & LPIPS$\downarrow$ & $\Delta$PSNR \\
\midrule
Full \softcap{} & 15.00 & 24.760 & 0.1319 & --- \\
w/o tail reserve & 15.00 & 18.205 & 0.4967 & $\mathbf{-6.55}$ \\
w/o PI feedback & 15.00 & 22.640 & 0.1667 & $-2.12$ \\
warmup $K{=}3$ & 15.00 & 22.730 & 0.1815 & $-2.03$ \\
w/o amplification, $A{\equiv}1$ & 15.00 & 24.196 & 0.1338 & $-0.56$ \\
w/o warm-start & 15.00 & 24.396 & 0.1360 & $-0.36$ \\
w/o age cap & 14.50 & 24.412 & 0.1403 & $-0.35$ \\
\bottomrule
\end{tabular}
\caption{Leave-one-out ablation at F15 (20 development prompts; FID omitted at $n{=}20$). Removing the tail reserve collapses quality at \emph{identical} Full count and FLOPs.}
\label{tab:ablation}
\end{table}

%% file: sec/5_conclusion.tex
\section{Conclusion}

We presented \softcap{}, a training-free control layer that turns cache refresh into budget-conditioned feedback. An error-amplification profile drives both the drift-risk observer and reference tracking, while a mandatory warmup and closed-form guards transfer one calibration across budget tiers. On FLUX.1-dev, \softcap{} leads Full-referenced fidelity among training-free baselines at every tier ($+2.7$~dB paired PSNR at the main point); paired-bootstrap and ablation results attribute the gains to Full-evaluation placement rather than added compute. On SD3.5 Large, risk-aware refreshes also substantially improve the same Taylor cache engine over its fixed-interval schedule. \softcap{} is a policy layer, not a replacement cache engine; broader engines and sampling settings remain future work.

\paragraph{Limitations.}
At the tightest budgets, mandatory warmup and safety guards leave fewer Full evaluations for prompt-adaptive placement, reducing sequence diversity. Therefore, we plan to further investigate mechanisms that preserve prompt adaptivity in highly constrained budget regimes.